%% file: LatentGroupReps_ICLR.tex
\documentclass{article} 
\usepackage{iclr2015,times}
\usepackage{hyperref}
\usepackage{url}

\usepackage{amssymb}
\usepackage{amsmath}
\usepackage{amsthm}
\usepackage{mathdots}
\usepackage{gensymb}
\usepackage{graphicx}
\usepackage{caption}
\usepackage{subcaption}
\usepackage{tikz}
\usetikzlibrary{bayesnet}

\newtheorem{thrm}{Theorem}

\title{Transformation Properties of Learned Visual Representations}

\author{
Taco S.~Cohen \& Max Welling \\
Machine Learning Group \\
Department of Computer Science\\
University of Amsterdam \\
\texttt{\{t.s.cohen, m.welling\}@uva.nl} \\
}

\iclrfinalcopy 

\iclrconference 

\begin{document}

\maketitle

\begin{abstract}

  When a three-dimensional object moves relative to an observer, a change occurs on the observer's image plane and in the visual representation computed by a learned model.
  Starting with the idea that a good visual representation is one that transforms linearly under scene motions,
  we show, using the theory of group representations, that any such representation is equivalent to a combination of the elementary \emph{irreducible} representations.
  We derive a striking relationship between irreducibility and the statistical dependency structure of the representation,
  by showing that under restricted conditions, irreducible representations are decorrelated.
  Under partial observability, as induced by the perspective projection of a scene onto the image plane, the motion group does not have a linear action on the space of images, 
  so that it becomes necessary to perform inference over a latent representation that does transform linearly.
  This idea is demonstrated in a model of rotating NORB objects that employs a latent representation of the non-commutative 3D rotation group $\textup{SO}(3)$.
  
\end{abstract}

\section{Introduction}

Much has been written about invariant representations (e.g. \citet{Anselmi2014}), and invariance to groups such as translations, rotations and projective transformations is indeed very important for object recognition.
However, for a general purpose visual representation -- capable not only of supporting recognition tasks but also motion understanding and geometrical reasoning -- invariance is not enough. 

Instead, the \emph{transformation properties} of a representation are crucially important.
If we could understand how a given representation of visual data transforms under various rigid or non-rigid transformations of the latent 3D scene, we would be in a better position to build an integrated system that computes invariant representations as well as motion and relative poses of objects.
However, performing a mathematical analysis of the transformation properties of, for example, the hidden layer in a deep neural network under motions of a 3D scene is extremely complicated.
A better approach
is to directly impose good transformation properties on a representation space, and then \emph{learn} the mapping between data and representation space such that these transformation properties are realized \citep{Hinton2011}.

In this paper we study the transformation properties of distributed representations, using tools from group representation theory \citep{Sugiura1976}.
We relate the transformation properties of a distributed representation to statistical notions such as decorrelation and conditional independence under the assumption of complete observability.
Under partial observability (due to occlusion, for example)
it becomes necessary to introduce latent variables in order to obtain a representation space with good transformation properties.
We propose a number of transformation properties that a good representation should have, and present a simple model that demonstrates the idea by modelling 3D rotations of objects from the NORB dataset~\citep{LeCun}.

Our model uses a single latent vector of coefficients to represent a set of images of the same object seen in different (rotated) poses, and uses one latent element of the 3D rotation group $\textup{SO}(3)$ for each pose.
A generative neural network model maps each transformed latent representation to an image.
Unlike previous work on learning group representations~\citep{Rao1999, Miao2007, Sohl-Dickstein2010, Wang2011, Bruna2013, Cohen2014}, our model does not assume a linear action of the group in the input space, but instead acts linearly on a latent representation of the 3D scene.
Furthermore, our model is the first learned Lie group model that can properly deal with non-commutative transformations.

The rest of the paper is organized as follows.
In the next section, we introduce the concept of a group representation which is at the core of our analysis.
Section three contains the main theoretical results on the dependency structure of irreducible representations.
It is followed by a discussion of the problems that arise when partial observability is taken into account in section 4.
Section 5 presents a model and training algorithm for learning latent group representations, followed by experiments, related work and a conclusion.

\section{Symmetries and Representations}
\label{sec:symmetries_and_representations}

We start from the basic assumption that our learning agent is situated in space, and this space contains a scene.
Formally, we represent the scene as a function $x : \mathbb{R}^3 \rightarrow \mathbb{R}^K$ that at each point $p$ in space gives a list of numbers $x(p)$ describing, for example, the color, transparency value, material properties, etc. at $p$.
In this section and the next, we further assume full observability, i.e. that $x$ is known entirely.
We think of $x$ as a vector in a Hilbert space $\mathcal{S}$ of sufficiently well-behaved functions.
As we will see, the following analysis
does not depend on this particular data representation,
but it provides useful intuition and is ultimately realistic.

We say that the vector $x$ is a \emph{representation} of the scene, because
the numerical values that one would store in a computer to describe or approximate $x$ depend on both ``what is in the scene'' \emph{and} ``how it is represented in our preferred frame of reference''.
If we transform our reference frame by $g$, an element of the special Euclidean group $\textup{SE}(3)$ of rigid body motions, the points in space transform as $g^{-1} p$.
Such a transformation leaves invariant Euclidean distances, angles and areas, and is therefore called a symmetry of Euclidean space.
Under this symmetry, the scene transforms as
\begin{equation}
  \label{eq:representation_on_hilbertspace}
  x'(p) = x(g^{-1} p) \equiv [T(g) x](p),
\end{equation}
Notice that $T(g)[\alpha x + \beta y](p) = \alpha [T(g) x](p) + \beta [T(g) y](p)$, so $T(g)$ is a linear operator.
We say that $T$ is a representation of $\textup{SE}(3)$ in the Hilbert space $\mathcal{S}$.

Generically, a group representation is a map $T : G \rightarrow \textup{GL}(V)$ from a group $G$ to the set of invertible linear transformations $\textup{GL}(V)$ on a vector space $V$, that preserves the group structure in the following sense:
\begin{equation}
  \label{eq:homomorphism}
  T(g) T(h) = T(gh),
\end{equation}
for all $g, h \in G$.
One can check that the map $T$ defined in eq. \ref{eq:representation_on_hilbertspace} is indeed a group representation.

The requirement that $(T, V)$ forms a representation of $\textup{SE}(3)$ is a sufficient condition for the vectors in $V$ to describe ``a thing in space'', because it requires them to transform as space does~\citep{Kanatani1990}.
This is true in particular for the Hilbert space construction given above, but applies more generally to any learned or hand-designed vector space representation.
Should eq. \ref{eq:homomorphism} fail to hold, key aspects of what it means to transform as Euclidean space are lost: for example, two $180\degree$ rotations about the same axis might not equal the identity transformation, or two translations might fail to commute.

Hence, we want our representation (in the representation learning sense) to be a linear representation of the special Euclidean group (in the group representation theory sense).
From a modelling perspective, one would like to understand all the possible ways in which this can be achieved.
To this end, observe that if we have a representation $T(g)$ and an invertible matrix $F$, then $T'(g) = F T(g) F^{-1}$ is also a representation, which is said to be equivalent to $T$.
A key result in group representation theory tells us that every unitary representation is equivalent in this sense to a simple composition of basic building blocks called irreducible representations.

A representation $(T, V)$ is called irreducible if there is no nontrivial subspace $W \subset V$ that is mapped onto itself by all operators $T(g)$ for $g \in G$.
It can be shown~\citep{Sugiura1976} that unitary representations are fully reducible, which means that the representation $T(g)$ is equivalent to a block-diagonal representation $\hat{T}(g)$ whose blocks are irreducible.
Such a block-diagonal representation $\hat{T} = F T(g) F^{-1}$ is said to be \emph{fully reduced}.
Each block in $\hat{T}(g)$ is identified by an index which we denote by $l$, so that we can write $\hat{T}^l(g)$ for a block of index $l$ in $\hat{T}(g)$ and $x^l$ for the component of $x \in V$ in the corresponding subspace.

Irreducible representations are important for representation learning in a number of ways.
Firstly, using irreducible representations is computationally more efficient than using reducible ones.
Irreducibility can also be used to define precisely what a ``disentangled'' representation is (\citet{Cohen2014}; see also \citet{Bengio2013}), and as shown in the next section, such a representation will have a simple dependency structure when certain conditions are met.
Finally, it is easier to compute a set of generators for the ring of polynomial invariants in an irreducible representation, which can be used to build invariant representations.
\citet{Kazhdan2003} use a subset of generators (the power spectrum) to build invariant (but lossy) shape descriptors.

\section{Irreducibility, Independence and Decorrelation}

Representation learning is often seen as a form of generative modelling, where the goal is to learn a latent variable model with a simple dependency structure in the latent space.
The simplest examples are PCA and ICA, where the goal is to learn a linear model whose latent variables are all independent (with Gaussian or non-Gaussian marginal distributions, respectively).

Alternatively, one can put the the transformation properties center stage and learn a representation that transforms irreducibly under symmetry transformations \citep{Cohen2014}.
In this perspective, the irreducible representations (and not the independent factors) are the elementary parts from which observation vectors are constructed.
Given these contrasting conceptualizations of representation learning, it is interesting to investigate how the transformation properties of a representation are related its statistical properties.
In this section we show that under certain conditions, irreducible representations are decorrelated or even conditionally independent.

\subsection{Irreducibility and decorrelation: an elementary example}
In order to gain some intuition, we introduce a simple toy model of a completely observable system with symmetry.
The states of the system are sufficiently well-behaved functions $x : \mathbb{S} \rightarrow \mathbb{R}$ on the circle.
Observations are generated by sampling a uniformly distributed rotation angle $\theta \in [0, 2\pi)$ and using it to rotate a template $\tau$.
That is, we have observations $x(\varphi) = [T(\theta) \tau](\varphi) = \tau(\varphi - \theta)$ for $\theta \sim \mathcal{U}[0, 2 \pi)$.
In practice, we will observe discretized functions with a finite number of coefficients $x_n = x(\varphi_n)$.
  
In this case, the linear transformation $F$ that achieves the reduction into irreducible representations is the standard Fourier transform, and indeed it will decorrelate the data~\citep{Bruna2013}.
To see this, let $\hat{x} = F x$, and observe that
\begin{equation}
  \hat{x} = F T(\theta) \tau = F T(\theta) F^{-1} \hat{\tau} \equiv \hat{T}(\theta) \hat{\tau}.
\end{equation}
Thus $\hat{T} = F T(\theta) F^{-1}$ is the representation of the rotation group in the spectral domain.

We know from linear algebra that a set of commuting diagonalizable matrices can be simultaneously diagonalized (see \citet{Memisevic2012} and \citet{Henriques2014a} for a discussion).
Hence, the fully reduced representation $\hat{T}$ is diagonal (not just block-diagonal), and the irreducible representations are one-dimensional.
The diagonal elements are complex exponentials $T_{ll}(\theta) = \exp{(i l \theta)}$.
It follows immediately that the covariance matrix of the Fourier-transformed data is diagonal:
\begin{equation}
  \mathbb{E}_{p(\theta)}[\hat{x}_l \, \hat{x}_{l'}^*]
  =
  \int_0^{2\pi} e^{i l \theta} \, \hat{\tau}_l \; e^{-i l' \theta} \, \hat{\tau}_{l'}^* \, \frac{d\theta}{2 \pi}
  =
  \delta_{ll'} \, |\hat{\tau}_l|^2,
\end{equation}
where $\delta_{ll'}$ equals $1$ if $l=l'$ and $0$ otherwise.

\subsection{Irreducibility and decorrelation: general case}

The following theorem gives a generalization of this result to the case of compact but not necessarily commutative groups.

\begin{thrm}
  \label{thm:diagonal_covariance}
  Let $G$ be a compact group, $V$ a real vector space,
  and $\hat{T}$ a fully reduced unitary representation of $G$ in $V$.
  Furthermore, let $\hat{x} = \hat{T}(g) \hat{\tau}$ for a fixed template $\hat{\tau} \in V$ and $g$ distributed uniformly on $G$.
  The covariance matrix of the vectors in $V$ is diagonal:
  $$
    \mathbb{E}_{p(g)}\left[\hat{x}^l_{m} \hat{x}^{l'}_{m'}\right]
    =
    \delta_{ll'} \delta_{mm'} \frac{\|\hat{\tau}^l\|^2}{\dim \hat{T}^l}.
  $$
\end{thrm}
\begin{proof} Using orthogonality of the matrix elements of irreducible representations. See appendix. \end{proof}

The theorem is easily generalized to more than one template $\tau$ (in which case one should consider the class-conditional covariance), and it is likely that a slightly weaker theorem can be proven for locally compact groups, but we will not do so here.
The main concern regarding the applicability of the above result is not the type of groups and spaces it applies to, but the fact that in reality the orbits are not sampled uniformly.
For example, in a sample of natural images, a human face is more likely to appear in upright position than upside down.

While the assumption of uniform sampling of orbits will not hold exactly in real datasets, it is nevertheless likely that irreducible and decorrelated representations will be similar to the degree that the data density is invariant to the group under consideration.
A statistical objective such as decorrelation makes sense when one is working with iid draws from an underlying distribution of images, but this is a rather impoverished model of visual experience.
As such, we think of decorrelation and independence as surrogate objectives for a deeper structural objective such as irreducibility.

\subsection{Irreducibility and Conditional Independence}

The concept of an irreducible representation can also shed light on time-series models based on transformations~\citep{Cohen2014, Michalski2014}.
We define
\begin{equation}
  p(x_t \; | \; x_{t-1}, g) = \mathcal{N}(x_t \; | \; T(g) x_{t-1}, \sigma^2),
\end{equation}
where $x_t$ and $x_{t-1}$ are observation vectors at times $t$ and $t-1$, respectively, and $T(g)$ is a unitary representation of a compact group $G$.
For $G$, one can construct an exponential family whose sufficient statistics are given by the matrix elements $\hat{T}^l_{mn}(g)$ of irreducible unitary representations of $G$.
As shown in~\citep{Cohen2014} for the case of compact commutative groups, the invariance of the $l_2$-norm in the exponent of the Gaussian to unitary transformations results in a posterior $p(g | x_t, x_{t-1})$ that is in the same exponential family as the prior $p(g)$ (conjugacy).
Furthermore, the marginal $p(x_t \; | \; x_{t-1})$ will factorize according the irreducible representations: $p(\hat{x}_t \; | \; \hat{x}_{t-1}) = \prod_l p^l(\hat{x}^l_t \; | \; \hat{x}^l_{t-1})$, so in this model an irreducible representation gives us conditional independence.

\section{Partial Observability}
\label{sec:partial_observability}

In reality, we do not observe the complete scene $x \in \mathcal{S}$ but only a projected image $I \in \mathcal{I}$, which we model as a function $I : \mathbb{R}^2 \rightarrow \mathbb{R}^3$ (for 3 color channels).
Naively, one could try to construct a representation $\bar{T} : \textup{SE}(3) \times \mathcal{I} \rightarrow \mathcal{I}$ such that the perspective projection $\pi : \mathcal{S} \rightarrow \mathcal{I}$ is an equivariant map: $\bar{T}(g) \circ \pi = \pi \circ T(g)$, but this is not possible.
The reason is that a 3D motion can bring entirely new structures into the image.

In classical computer vision, the solution is sought in strong assumptions on the scene geometry, such as the assumption that the scene is planar, in which case one obtains a representation of the projective group on the image plane.
This assumption leads to neat formulas but real scenes are not flat.
A better approach to the problem of partial observability is to model all variability that is not caused by the linear action of a low-dimensional Lie group as being caused by the action of the infinite-dimensional group of diffeomorphisms~\citep{Bruna2013, Soatto2012}.

The scattering representations of~\citet{Bruna2013_TPAMI} achieve simultaneous insensitivity to translations and diffeomorphisms,
and this method achieves very good performance on texture recognition and 2D pattern recognition (e.g. MNIST).
However, diffeomorphisms are not an entirely satisfactory model of projected 3D motions either,
because they are invertible by definition while projected motions are not.
Furthermore, arbitrarily small scene motions can bring arbitrarily bright structures into the image, so scattering representations are not Lipschitz continuous to scene motions.

What we can do instead (at least in principle) is to learn a prior over scenes and a generative model of images given scenes, and then perform inference over scenes given images.
By requiring that the latent scene transforms as a representation of a symmetry group, we bias the model towards representing latent properties of the \emph{scene} as opposed to properties of the \emph{image}~\citep{Kanatani1990, Soatto2012}.
By further requiring that the latent scene transforms irreducibly, we may also obtain a simple dependency structure in the latent space (by theorem \ref{thm:diagonal_covariance}).

\section{A Latent Group Representation}

In this section we define a simple model that demonstrates the idea of a latent group representation concretely.
Let $X^n = [x^{n,1}, \ldots, x^{n,V}]$ be a matrix of $V$ views of the same object instance $n$.
We model such a set of views using a single latent vector $z^n$ and one latent transformation $g^{n,v} \in \textup{SO}(3)$ per view, which we collect in a matrix $G^n = [g^{n,1}, \ldots, g^{n,V}]$.
In order to generate $x^{n,v}$, we first compute $z^{n,v} = \hat{T}(g^{n,v}) z^{n}$ (this computation is explained in section \ref{sec:computation_of_the_representation_matrices}) and then pass this transformed latent scene to a neural network.
The conditional $p(x^{n,v} \, | \, z^n, g^{n,v})$ is given by a normal distribution, centered on the output of a generative neural network $f_{\theta} : \mathbb{R}^{D_z} \rightarrow \mathbb{R}^{D_x}$ that maps $z$ to $x$-space: 

\begin{equation}
p(x^{n,v} \, | \, z^n, \,g^{n,v}) = \mathcal{N}(x^{n,v} \, | \, f_{\theta}(\hat{T}(g^{n,v}) \, z^n), \, \sigma_x^2)
\end{equation}

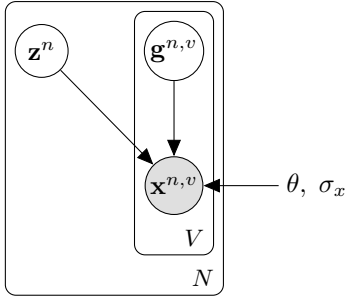
\begin{figure}[t!]
  \begin{center}
      \input{graphical_model}
  \end{center}
  \caption{Graphical model description of the latent linear $\textup{SO}(3)$ representation.}
  \label{fig:graphical_model}
\end{figure}

We use a standard normal prior on $z$, and a uniform distribution over $\textup{SO}(3)$ for $g$.
The complete graphical model is shown in figure~\ref{fig:graphical_model}.
For regularization, we use a zero-mean Gaussian prior on the neural network weights and on $\ln \sigma_x$ (with precision $\beta$ and $\alpha$, respectively).
The complete log joint probability for a single instance is then given by:
\begin{equation}
\begin{aligned}
\ln p(X^n, G^n, z^n, \theta, \sigma_x)
&=
&-& \sum_{v=1}^V \frac{\|x^{n,v} - f_{\theta}(\hat{T}(g^{n,v}) z^n)\|^2}
                {2 \sigma_x^2}
- \frac{V D_x}{2} \ln \sigma_x \\
&&-&\frac{\|z^n\|^2}{2} - \beta \frac{ \|\theta\|^2}{2} - \alpha \frac{(\ln \sigma_x)^2}{2}
\end{aligned}
\end{equation}
The gradients of which are easily computed using backpropagation (in our own implementation, we compute gradients automatically using Theano~\citep{Bergstra2010}).

\subsection{Representation Theory of $\textup{SO}(3)$}
\label{sec:representation_theory_of_so3}

In order to compute the transformed latent scene $z^{n,v} = \hat{T}(g^{n,v}) z^n$, we must understand the structure of the unitary representation $\hat{T}$, and find out how to compute it.
Since any representation is equivalent to a block diagonal one, we take $\hat{T}$ to be block-diagonal:
\begin{equation}
  \hat{T}(g) =
    \begin{bmatrix}
    \hat{T}^{l_1}(g) &        & \\
                    & \ddots & \\
                    &        & \hat{T}^{l_{N}}(g)
  \end{bmatrix},
\end{equation}
where $\hat{T}^l$ is the matrix of an irreducible representation of index $l$.

The complete set of irreducible unitary representations of $\textup{SO}(3)$ (the \emph{unitary dual}) can be obtained by decomposing what is called the \emph{regular representation} of $\textup{SO}(3)$ acting on functions on the sphere $\mathbb{S}_2$.
In this case, the representation space is the Hilbert space $\mathcal{H}$ of square-integrable functions on the sphere, and the representation is defined as $T(g) x(p) = x(g^{-1} p)$ for $g \in \textup{SO}(3), \, x \in \mathcal{H}, \, p \in \mathbb{S}_2$.
A function  $x \in \mathcal{H}$ can be decomposed as a sum of so-called real spherical harmonic functions $Y_{lm} : \mathbb{S}_2 \rightarrow \mathbb{R}$ (for $l \geq 0, |m| \leq l$).
That is, for any $x \in \mathcal{H}$ we can write
\begin{equation}
  x(p)=\sum_{l\geq 0}\sum_{m=-l}^l c_{lm} Y_{lm}(p),
\end{equation}
for some $x$-dependent coefficients $c_{lm}$ (comparable to ordinary Fourier coefficients of a periodic function on the line).
The matrix elements of the representation $T$ in this basis are given by
\begin{equation}
  \hat{T}^l_{mn}(g) = \langle Y_{lm}, T(g) Y_{ln} \rangle = \int_{\mathbb{S}_2} Y_{lm}(p) [T(g) Y_{ln}](p) dp = \int_{\mathbb{S}_2} Y_{lm}(p) Y_{ln}(g^{-1} p) dp.
\end{equation}
It can be shown \citep{Sugiura1976} that this representation $\hat{T}^l(g)$, which maps coefficients $(c_{l,-l}, \ldots, c_{l,l})$ to coefficients $(c'_{l,-l}, \ldots, c'_{l,l})$ corresponding to the expansion of the rotated function $x'(p) = x(g^{-1} p)$, is irreducible.
Furthermore, all irreducible unitary representations of $\textup{SO}(3)$ are equivalent to some $\hat{T}^l$ obtained in this way.

To get an intuitive understanding of the transformation properties of the spherical harmonics, consider figure \ref{fig:sh}:
the basis functions on each row (corresponding to one value for $l$) can be linearly combined, and any rotation of the resulting function can again be expressed as a linear combination of only those basis functions.
That is, each row corresponds to a representation.

For the experiments detailed in section~\ref{sec:experiments}, we will be interested in the action of $\textup{SO}(3)$ on 3D objects, which could be represented as functions of some compact region of 3D space.
Such a function can be decomposed by using multiple copies of each irreducible representation, as was done in~\citet{Skibbe2009} in the context of rotation invariant shape descriptors.

\subsection{Computation of the Representation Matrices}
\label{sec:computation_of_the_representation_matrices}

We now turn to the computation of the transformation $x \rightarrow \hat{T}(g) x$.
Due to the block-structure of $\hat{T}$, this computation breaks up into a large number of relatively small matrix multiplies.
The matrix elements $\hat{T}^l_{mn}$ of the irreducible representations of $\textup{SO}(3)$ are known as the Wigner D-functions.
The formulae given for these matrix elements by~\citet{Wigner1960} involve numerically unstable sums of many elements with large coefficients.
Quite surprisingly, given the prominence of these matrices in physical theories and their long history, a relatively recent paper introduced a novel and very fast method for computing the representation matrices in the basis of real spherical harmonics~\citep{Pinchon2007}.
The authors of this paper show that in the basis of real spherical harmonics, a rotation specified by ZYZ-Euler angles $g = (g_1, g_2, g_3)^T$ can be computed as $\hat{T}^l(g) = \hat{T}^l_z(g_3) J^l \hat{T}^l_z(g_2) J^l \hat{T}^l_z(g_1)$, where $J^l$ is a precomputed symmetric orthogonal block matrix that exchanges the Y and Z axes, and $T^l_z$ represents a $z$-axis rotation, which takes the simple form:
\begin{equation}
\hat{T}_z^l(\alpha) = 
\begin{bmatrix}
\cos(l \alpha)  &                    &         & &         &                      & \sin(l \alpha) \\
                & \cos((l-1)\alpha)  &         & &         & \sin((l - 1) \alpha) & \\ 
                &                    & \ddots  & & \iddots &                      & \\
                &                    &         &1&         &                      & \\
                &                    & \iddots & & \ddots  &                      & \\
                & -\sin((l-1)\alpha) &         & &         & \cos((l-1) \alpha)   & \\
-\sin(l \alpha) &                    &         & &         &                      & \cos(l \alpha) \\
\end{bmatrix}
\end{equation}
Figure~\ref{fig:dmats} shows the matrix $\hat{T}^{35}$ corresponding to weight $l=35$ for three values of $g$.

\begin{figure}
        \centering
        \begin{subfigure}{0.3\textwidth}
                \includegraphics[width=\textwidth]{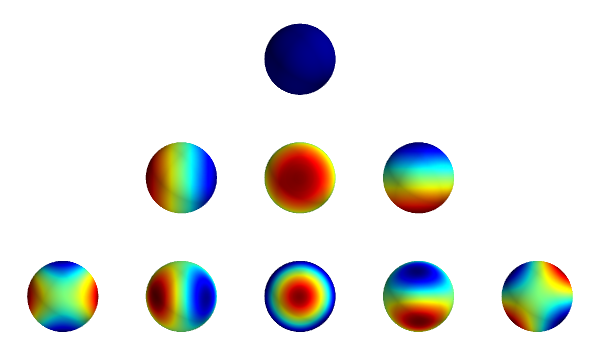}
                \caption{Real Spherical Harmonics}
                \label{fig:sh}
        \end{subfigure}
        \begin{subfigure}{0.5\textwidth}
                \includegraphics[width=\textwidth]{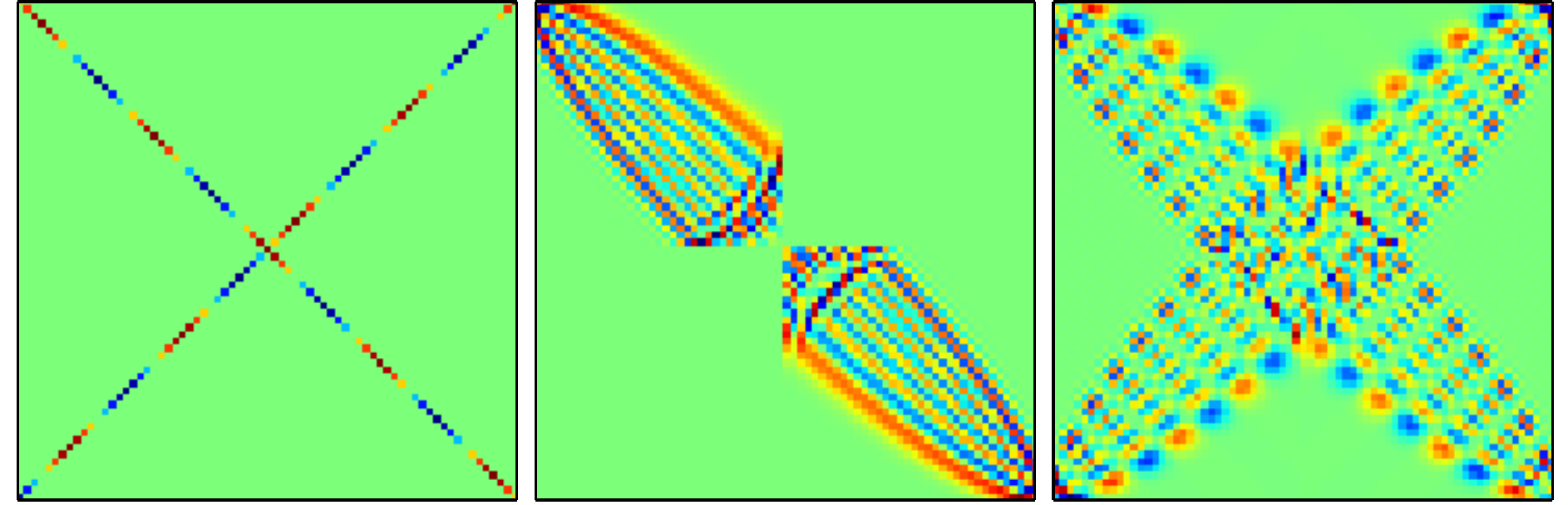}  
                \caption{Real Wigner-D matrices for $\mathbf{T}^{35}(\pi/8,0,0) = \mathbf{T}_z^{35}(\pi/8)$, $\mathbf{T}^{35}(0,\pi/8,0)$, and $\mathbf{T}^{35}(\pi/8,\pi/8,\pi/8)$}
                \label{fig:dmats}
        \end{subfigure}%
        \caption{Real spherical harmonics and Wigner D-Matrices}
        \label{fig:dmats_and_sh}
\end{figure}

Naively implemented, this method has computational complexity $O(l^3)$ in dimension $2l + 1$, due to the matrix multiplications.
However, it is possible to \emph{apply} the matrix $\hat{T}^l(g)$ to a vector without explicitly constructing it, using associativity: $x' = \hat{T}(g) x = \hat{T}_z(g_3) (J (\hat{T}_z(g_2) (J (\hat{T}_z(g_1) x))))$.
The sparse multiplication $\hat{T}_z x$ takes linear time, while $J x$ takes quadratic time.
In practice, we use many copies of relatively low-dimensional representations, so the values of $l$ are much smaller than the dimensionality of the latent space, and hence the quadratic complexity is not a concern.

\subsection{Learning}
\label{sec:learning}

We train the model using a stochastic hard EM algorithm, which involves alternating between the following steps:
\begin{enumerate}
\item In hard EM, the E-step consists of partial maximization with respect to $z^n$ and $g^{n,v} \, (v=1,\ldots,V)$ for a single instance $n$ while keeping the parameters $\theta, \sigma_x$ fixed. We initialize the latent variables at the state of the last iteration and perform one step of gradient ascent on
$\ln p(X^{n}, \, G^{n}, \, z^n, \, \theta, \sigma_x)$.
\item The M-step consists of a maximization with respect to the parameters $\theta, \sigma_x$ while holding latent variables fixed.
In our stochastic algorithm, we perform a single gradient step on $\ln p(X^{n}, \, G^{n}, \, z^n, \, \theta, \sigma_x)$. 
\end{enumerate}

We use adagrad for all optimization~\citep{Duchi2011}.

\section{Experiments}
\label{sec:experiments}

We trained the model on the NORB dataset~\citep{LeCun}.
This dataset consists of objects in $5$ generic categories: four-legged animals, human figures, airplanes, trucks, and cars.
Each category contains $10$ instances, of which we used the last $5$ for training.
Each instance is imaged at $9$ camera elevations (30 to 70 degrees from horizontal, in 5 degree increments) and 18 azimuths (0 to 340 degrees in 20 degree increments).
Finally, there are $6$ lighting conditions for each instance, yielding a total of $5 \cdot 5 \cdot 6 \cdot 9 \cdot 18 = 24300$ images.

The data was made zero mean, contrast normalized and then PCA whitened, retaining $95\%$ of the variance.
We used a neural network $f_\theta$ with one hidden layer containing 550 hidden units.
The group representation $\hat{T}$ is determined by a choice of $l_i; i=1,\ldots, L$ which we chose to be: $[0]\times 20 + [1] \times 15 + [2] \times 10 + [3] \times 10 + [4] \times 10 + [5] \times 9 + [6] \times 8 + [7] \times 7 + [8] \times 6 + [9] \times 5$, where the number in brackets represents $l_i$ and the multiplier denotes its multiplicity.
The regularization parameters were set to $\beta=0.1, \alpha=0.1$.

In figure \ref{fig:interpolation}, we show that the model is able to generate reasonable images for angles it has never seen before.
The model is only trained on images that are off by $20$ azimuthal degrees, but the model can produce images off by much smaller angles.

In figure \ref{fig:extrapolation}, we show that the model is able to extrapolate to unseen angles of a known object. That is, we train the model only on the azimuthal angle larger than 40 degrees from the reference figure (i.e. rotation 0), but produce a mean figure from the network at angles $0, 20, 40, 60$.
The model gives reasonable images that retain the object identity for poses in which it has not seen that object before.

\begin{figure}
        \centering
                \includegraphics[width=0.8\textwidth]{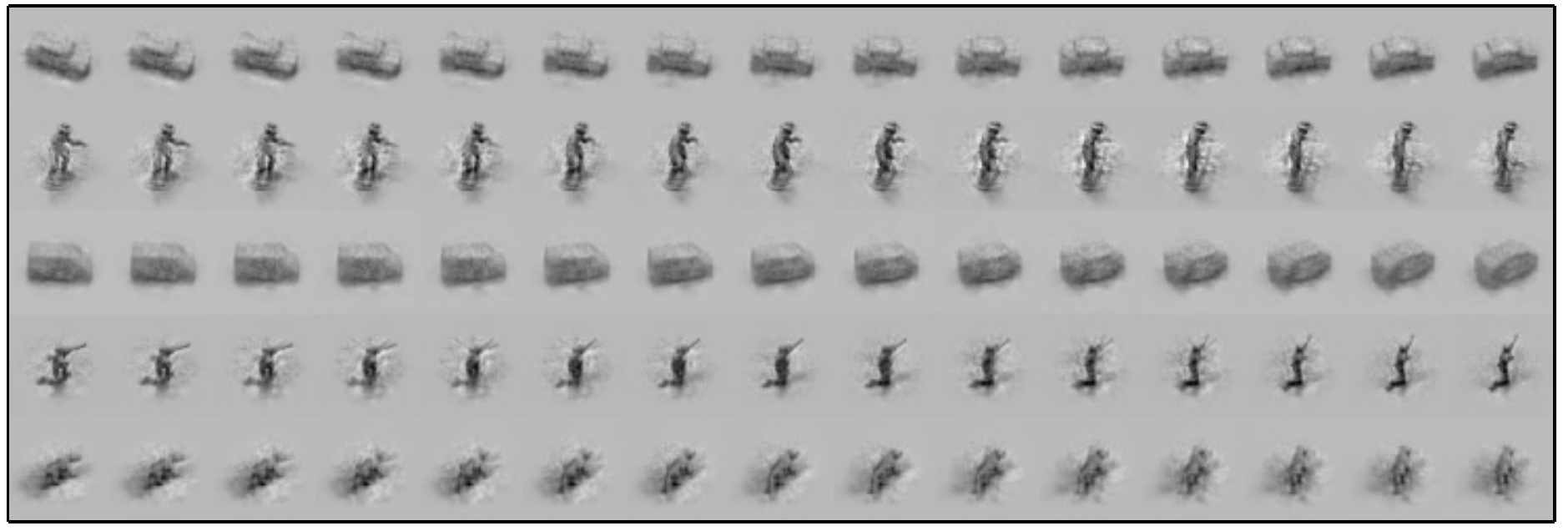}
        \caption{Interpolation over $40$ degrees for various objects.}
        \label{fig:interpolation}
\end{figure}

\begin{figure}
        \centering
                \includegraphics[width=0.7\textwidth]{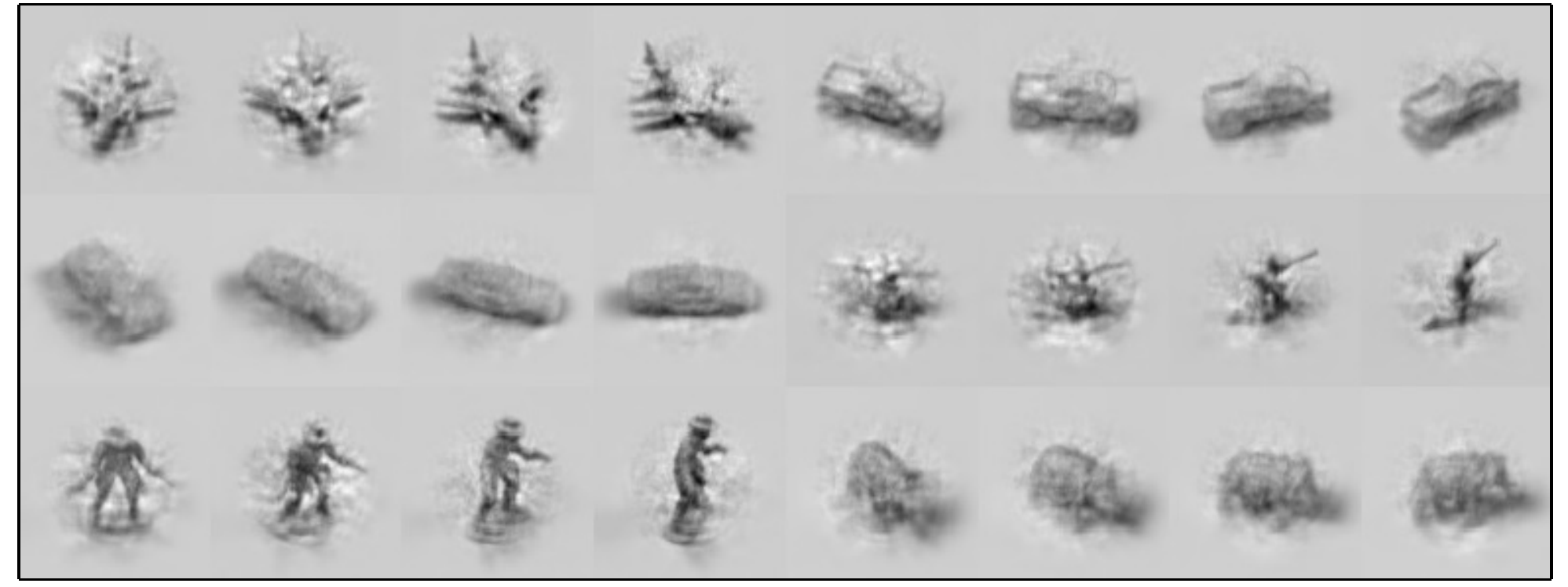}
        \caption{Extrapolation: the first two images in each sequence of 4 were not part of the training set.}
        \label{fig:extrapolation}
\end{figure}

\section{Related Work}
\label{sec:related_work}

Our work is related to the idea of transforming auto-encoders or ``capsules'' by~\citet{Hinton2011}.
A transforming auto-encoder consists of many capsules, each of which learns to recognize a visual entity and predict its pose.
The pose variables $g$ are thus explicitly represented in the model, and act linearly on other pose variables, as is the case in our model.
Unlike our model, a transforming auto-encoder represents the scene content as a set of probabilities, each of which indicates the likelihood of the preferred visual entity being present.

The binary recognition unit used by a capsule for object $z$ corresponds to an orbit $O(z) = \{\hat{T}(g) z \; | \; g \in G \}$ in our model.
Having a single latent space shared by multiple visual entities may aid in generalization, and makes it possible to compute metric relations between different objects.
Our approach should also deal better with (approximately) symmetric objects, for which it is not possible to unambiguously estimate pose and motion (what is the pose of a circle?).
For the case of translational motion of an edge-like structure, this is known as the aperture problem~\citep{Memisevic2012}.
Instead of trying to estimate the motion anyway, our model would represent the edge as a vector whose orbit has reduced dimensionality compared to non-symmetric objects.

That said, the fully connected generative network and hard-EM algorithm used in our current model are not suitable for dealing with large images, and so we consider the current model as only a proof of concept.
A more scalable linear representation learning system could be based on a group-invariant convolutional network (e.g.~\citet{Gensa, Mallat2012}) that generates a distributed representation that at each point locally describes the scene content, while transforming in a (locally) covariant manner.

\pagebreak
\section{Conclusion}
\label{sec:conclusion}

As the problem of object recognition in static images is steadily approaching the ``solved'' status, we should start looking towards the next frontier.
One of the central challenges is to move away from the idea that images are i.i.d. draws from an underlying distribution (as is the case only in current day benchmark datasets), and begin to model the dynamics of the visual world.
Another challenge is to generalize effectively from few examples, which necessitates the exploitation of symmetries of the data distribution.
Both of these problems require us to take a closer look at the transformation properties of learned visual representations.

In this paper, we have theoretically studied the consequences of assuming a linear representation of a symmetry group in the observed or latent representation space.
We have shown that the entire class of such models can be understood mathematically (they are all direct sums of irreducible representations), and have shown how the theory specializes for the case of the 3D rotation group.
Furthermore, we have shown that under uniform sampling of orbits, the geometrical objective of learning a linear, unitary and irreducible representation leads to decorrelated representations, thereby shedding new light on this common learning objective.

\subsubsection*{Acknowledgments}
This work was supported by NWO, grant number NAI.14.108.

\bibliography{LatentGroupReps_ICLR}
\bibliographystyle{iclr2015}

\section{Appendix}

\subsection{Proof of Theorem \ref{thm:diagonal_covariance}}

We have $\hat{x} = \hat{T}(g) \hat{\tau}$, and $g$ distributed uniformly.
Let $\mu(g)$ denote the normalized Haar measure on $G$.

\begin{equation}
  \begin{aligned}
    \mathbb{E}_{p(g)}\left[\hat{x}^l_{m} \hat{x}^{l'}_{m'}\right]
    &=
    \int_G \left(\sum_n \hat{T}^l_{mn}(g) \hat{\tau}^l_n\right)
    \cdot
    \left(\sum_{n'} \hat{T}^{l'}_{m'n'}(g) \hat{\tau}^{l'}_{n'}\right) d\mu(g) \\
    &=
    \sum_{nn'} \hat{\tau}^l_n \hat{\tau}^{l'}_{n'} \int_G \hat{T}^l_{mn}(g) \hat{T}^{l'}_{m'n'}(g) d\mu(g) \\
    &=
    \sum_{nn'} \hat{\tau}^l_n \hat{\tau}^{l'}_{n'} \frac{\delta_{ll'} \delta_{mm'} \delta_{nn'}}{\dim \hat{T}^l}\\
    &=
    \sum_{n} \hat{\tau}^l_n \hat{\tau}^{l}_{n} \frac{\delta_{ll'} \delta_{mm'}}{\dim \hat{T}^l} \\
    &=
    \frac{ \|\hat{\tau}^l\|^2}{\dim \hat{T}^l} \, \delta_{ll'} \delta_{mm'} \\
  \end{aligned}
\end{equation}

Where we used the orthogonality of matrix elements of irreducible representations~\citep{Sugiura1976}:
\begin{equation}
  \langle \hat{T}^l_{mn}, \hat{T}^{l'}_{m'n'} \rangle = \frac{\delta_{ll'} \delta_{mm'} \delta_{nn'}}{\dim \hat{T}^l},
\end{equation}
and where $\dim \hat{T}^l$ denotes the dimension of the representation indexed by $l$.

\end{document}

%% file: graphical_model.tex
%
%
%
%


\begin{tikzpicture} 

  \node[latent] (z) {$\mathbf{z}^n$} ;
  \node[latent, right=of z]  (g) {$\mathbf{g}^{n,v}$} ;
  \node[obs, below=of g] (x) {$\mathbf{x}^{n,v}$} ;
  \node[const, right=of x] (t) {$\; \theta, \; \sigma_x$};


  

  \edge {z, g, t} {x};


  \plate {views} {(g)(x)} {$V$} ;
  \plate {instances} {(z)(g)(x)(views.south east)(views.north east)} {$N$} ;

\end{tikzpicture}


%% file: LatentGroupReps_ICLR.bbl
\begin{thebibliography}{25}
\providecommand{\natexlab}[1]{#1}
\providecommand{\url}[1]{\texttt{#1}}
\expandafter\ifx\csname urlstyle\endcsname\relax
  \providecommand{\doi}[1]{doi: #1}\else
  \providecommand{\doi}{doi: \begingroup \urlstyle{rm}\Url}\fi

\bibitem[Anselmi et~al.(2014)Anselmi, Leibo, Rosasco, Mutch, Tacchetti, and
  Poggio]{Anselmi2014}
Anselmi, Fabio, Leibo, Joel~Z, Rosasco, Lorenzo, Mutch, Jim, Tacchetti, Andrea,
  and Poggio, Tomaso.
\newblock {Unsupervised learning of invariant representations with low sample
  complexity: the magic of sensory cortex or a new framework for machine
  learning?}
\newblock Technical Report 001, MIT Center for Brains, Minds and Machines,
  2014.

\bibitem[Bengio et~al.(2013)Bengio, Courville, and Vincent]{Bengio2013}
Bengio, Y., Courville, A., and Vincent, P.
\newblock {Representation Learning: A Review and New Perspectives}.
\newblock \emph{IEEE transactions on pattern analysis and machine
  intelligence}, pp.\  1--30, February 2013.

\bibitem[Bergstra et~al.(2010)Bergstra, Breuleux, Bastien, Lamblin, Pascanu,
  Desjardins, Turian, Warde-Farley, and Bengio]{Bergstra2010}
Bergstra, J., Breuleux, O., Bastien, F., Lamblin, P., Pascanu, R., Desjardins,
  G., Turian, J., Warde-Farley, D., and Bengio, Y.
\newblock {Theano: A CPU and GPU math compiler in Python}.
\newblock In \emph{Proceedings of the Python for Scientific Computing
  Conference (SciPy)}, pp.\  1--7, 2010.

\bibitem[Bruna \& Mallat(2013)Bruna and Mallat]{Bruna2013_TPAMI}
Bruna, Joan and Mallat, St\'{e}phane.
\newblock {Invariant scattering convolution networks}.
\newblock \emph{IEEE transactions on pattern analysis and machine
  intelligence}, 35\penalty0 (8):\penalty0 1872--86, August 2013.
\newblock ISSN 1939-3539.
\newblock \doi{10.1109/TPAMI.2012.230}.

\bibitem[Bruna et~al.(2013)Bruna, Szlam, and LeCun]{Bruna2013}
Bruna, Joan, Szlam, Arthur, and LeCun, Yann.
\newblock {Learning Stable Group Invariant Representations with Convolutional
  Networks}.
\newblock In \emph{International Conference on Learning Representations
  (ICLR)}, January 2013.

\bibitem[Cohen \& Welling(2014)Cohen and Welling]{Cohen2014}
Cohen, T. and Welling, M.
\newblock {Learning the Irreducible Representations of Commutative Lie Groups}.
\newblock In \emph{International Conference on Machine Learning (ICML)},
  volume~32, 2014.

\bibitem[Duchi et~al.(2011)Duchi, Hazan, and Singer]{Duchi2011}
Duchi, John, Hazan, Elad, and Singer, Yoram.
\newblock {Adaptive Subgradient Methods for Online Learning and Stochastic
  Optimization}.
\newblock \emph{Journal of Machine Learning Research}, 12(Jul):\penalty0
  2121--2159, 2011.

\bibitem[Gens \& Domingos(2014)Gens and Domingos]{Gensa}
Gens, Robert and Domingos, Pedro.
\newblock {Deep Symmetry Networks}.
\newblock In \emph{Advances in neural information processing systems}, 2014.

\bibitem[Henriques et~al.(2014)Henriques, Martins, Caseiro, and
  Batista]{Henriques2014a}
Henriques, Joao~F., Martins, Pedro, Caseiro, Rui, and Batista, Jorge.
\newblock {Fast Training of Pose Detectors in the Fourier Domain}.
\newblock In \emph{Advances in neural information processing systems}, 2014.

\bibitem[Hinton et~al.(2011)Hinton, Krizhevsky, and Wang]{Hinton2011}
Hinton, GE, Krizhevsky, A, and Wang, SD.
\newblock {Transforming auto-encoders}.
\newblock \emph{ICANN-11: International Conference on Artificial Neural
  Networks, Helsinki}, 2011.

\bibitem[Kanatani(1990)]{Kanatani1990}
Kanatani, K.
\newblock \emph{{Group Theoretical Methods in Image Understanding}}.
\newblock Springer-Verlag New York, Inc., Secaucus, NJ, USA, 1990.
\newblock ISBN 0387512535.

\bibitem[Kazhdan et~al.(2003)Kazhdan, Funkhouser, and
  Rusinkiewicz]{Kazhdan2003}
Kazhdan, Michael, Funkhouser, Thomas, and Rusinkiewicz, Szymon.
\newblock {Rotation invariant spherical harmonic representation of 3D shape
  descriptors}.
\newblock In \emph{Eurographics Symposium on Geometry Processing}, 2003.

\bibitem[LeCun \& Bottou(2004)LeCun and Bottou]{LeCun}
LeCun, Y. and Bottou, L.
\newblock {Learning methods for generic object recognition with invariance to
  pose and lighting}.
\newblock \emph{Proceedings of the 2004 IEEE Computer Society Conference on
  Computer Vision and Pattern Recognition, 2004. CVPR 2004.}, 2:\penalty0
  97--104, 2004.
\newblock \doi{10.1109/CVPR.2004.1315150}.

\bibitem[Mallat(2012)]{Mallat2012}
Mallat, Stephane.
\newblock {Group Invariant Scattering}.
\newblock \emph{Communications in Pure and Applied Mathematics}, 65\penalty0
  (10):\penalty0 1331--1398, 2012.

\bibitem[Memisevic(2012)]{Memisevic2012}
Memisevic, R.
\newblock {On multi-view feature learning}.
\newblock \emph{International Conference on Machine Learning}, 2012.

\bibitem[Miao \& Rao(2007)Miao and Rao]{Miao2007}
Miao, X. and Rao, R. P.~N.
\newblock {Learning the Lie groups of visual invariance.}
\newblock \emph{Neural computation}, 19\penalty0 (10):\penalty0 2665--93,
  October 2007.

\bibitem[Michalski et~al.(2014)Michalski, Memisevic, and Konda]{Michalski2014}
Michalski, Vincent, Memisevic, Roland, and Konda, K.
\newblock {Modeling Deep Temporal Dependencies with Recurrent Grammar Cells}.
\newblock In \emph{Advances in neural information processing systems}, pp.\
  1--9, 2014.

\bibitem[Pinchon \& Hoggan(2007)Pinchon and Hoggan]{Pinchon2007}
Pinchon, Didier and Hoggan, Philip~E.
\newblock {Rotation matrices for real spherical harmonics: general rotations of
  atomic orbitals in space-fixed axes}.
\newblock \emph{Journal of Physics A: Mathematical and Theoretical},
  40\penalty0 (7):\penalty0 1597--1610, February 2007.
\newblock ISSN 1751-8113.

\bibitem[Rao \& Ruderman(1999)Rao and Ruderman]{Rao1999}
Rao, R. P.~N. and Ruderman, D.~L.
\newblock {Learning Lie groups for invariant visual perception}.
\newblock \emph{Advances in neural information processing systems},
  816:\penalty0 810--816, 1999.

\bibitem[Skibbe et~al.(2009)Skibbe, Wang, and Reisert]{Skibbe2009}
Skibbe, Henrik, Wang, Qing, and Reisert, Marco.
\newblock {Fast computation of 3d spherical fourier harmonic descriptors - a
  complete orthonormal basis for a rotational invariant representation of
  three-dimensional objects}.
\newblock In \emph{IEEE International Workshop on 3-D Digital Imaging and
  Modeling (3DIM 2009), in conjunction with the ICCV 2009}, pp.\  1863----1869,
  2009.

\bibitem[Soatto(2012)]{Soatto2012}
Soatto, Stefano.
\newblock {Steps Toward a Theory of Visual Information: Active Perception,
  Signal-to-Symbol Conversion and the Interplay Between Sensing and Control}.
\newblock \emph{CoRR}, abs/1110.2, 2012.

\bibitem[Sohl-Dickstein et~al.(2010)Sohl-Dickstein, Wang, and
  Olshausen]{Sohl-Dickstein2010}
Sohl-Dickstein, J., Wang, J.~C., and Olshausen, B.~A.
\newblock {An unsupervised algorithm for learning lie group transformations}.
\newblock \emph{arXiv preprint}, 2010.

\bibitem[Sugiura(1990)]{Sugiura1976}
Sugiura, Mitsuo.
\newblock \emph{{Unitary Representations and Harmonic Analysis}}.
\newblock John Wiley \& Sons, New York, London, Sydney, Toronto, 2nd edition,
  1990.

\bibitem[Wang et~al.(2011)Wang, Shol-Dickstein, Tosic, and Olshausen]{Wang2011}
Wang, CM, Shol-Dickstein, J, Tosic, Ivana, and Olshausen, Bruno~A.
\newblock {Lie Group Transformation Models for Predictive Video Coding}.
\newblock \emph{Data Compression Conference}, pp.\  83--92, 2011.

\bibitem[Wigner(1959)]{Wigner1960}
Wigner, E.~P.
\newblock {Group Theory and Its Application to the Quantum Mechanics of Atomic
  Spectra}, 1959.
\newblock ISSN 00029505.

\end{thebibliography}
